\documentclass[letterpaper, 10 pt, conference]{ieeeconf}  

\IEEEoverridecommandlockouts                              

\usepackage{cite}
\usepackage{amsmath,amssymb,amsfonts}
\usepackage{algorithmic}
\usepackage{graphicx}
\usepackage{textcomp}
\usepackage{verbatim}
\usepackage[linesnumbered,ruled,vlined]{algorithm2e}
\usepackage{xcolor}
\usepackage{subfig}

\newtheorem{theorem}{\bf Theorem}[section]
\newtheorem{proposition}{Proposition}[section]

\DeclareMathOperator{\E}{\mathbb{E}}
\DeclareMathOperator*{\argmax}{arg\,max}
\DeclareMathOperator*{\argmin}{arg\,min}

\title{\LARGE \bf
Reinforcement Learning for Vision-based Object Manipulation with Non-parametric Policy and Action Primitives
}

\author{Dongwon Son$^*$, Myungsin Kim, Jaecheol Sim, and Wonsik Shin 
\thanks{
The authors are with Simulation Lab, Samsung Research, Samsung Electronics, Seoul, Republic of Korea\newline
$^*$: Corresponding author, {\tt\small dongwon.son@samsung.com}
}%
}

\begin{document}

\maketitle
\thispagestyle{empty}
\pagestyle{empty}

\begin{abstract}
The object manipulation is a crucial ability for a service robot, but it is hard to solve with reinforcement learning due to some reasons such as sample efficiency.
In this paper, to tackle this object manipulation, we propose a novel framework, AP-NPQL (Non-Parametric Q Learning with Action Primitives), 
that can efficiently solve the object manipulation with visual input and sparse reward,
by utilizing a non-parametric policy for reinforcement learning and appropriate behavior prior for the object manipulation.
We evaluate the efficiency and the performance of the proposed AP-NPQL
for four object manipulation tasks on simulation (pushing plate, stacking box, flipping cup, and picking and placing plate),
and it turns out that our AP-NPQL outperforms the state-of-the-art algorithms based on parametric policy and behavior prior in terms of learning time and task success rate.
We also successfully transfer and validate the learned policy of the plate pick-and-place task
to the real robot in a sim-to-real manner.

\end{abstract}
\section{INTRODUCTION}


Object manipulation is essential for robots who are working for e.g., household chores such as setting and clearing the table, tidying up toys, etc.
However, achieving this object manipulation, especially from visual input, is difficult since it requires robots to determine which object to grasp, where to grasp, when to grasp and release from the partially occluded visual input.
One promising direction of solving this object manipulation is reinforcement learning (RL), which has shown promise of success for various problems 
such as games \cite{silver2016mastering, mnih2015human}, 
gait control of quadruped\cite{hwangbo2019learning}, and dexterous in-hand manipulation \cite{akkaya2019solving}.
Despite recent advance in RL, it still suffers from both low sample efficiency and performance drop caused by local optima. It even becomes more difficult for the object manipulation task, due to
high dimensional visual sensory input, 
continuous action space, and
long-horizon, sparse reward characteristic.
To overcome those challenges, 
there have been many researches such as, 
decomposing vision module \cite{vecerik2019practical}, 
defining sub goals \cite{riedmiller2018learning}, 
engineering reward \cite{wulfmeier2020data},
formulating as regularized Markov Decision Process (MDP) \cite{abdolmaleki2018maximum, wu2019behavior, nair2020accelerating, siegel2020keep, tirumala2020behavior}, and etc.

Among these approaches, 
the recent progress of the regularized MDP formulation, 
which makes policy keep close to the desirable distribution during training process,
shows to accelerate the process and improve final performance even on the object manipulation task
without reducing the degree-of-freedom (DOF) of action \cite{kalashnikov2018qt, akinola2020learning}.
However, the previous works utilizing this regularized MDP adopt the policy iteration scheme with a parametric policy,
where the parameters of the policy are updated by projection of optimal non-parametric policy \cite{ghosh2020operator},
and we found that it degrades the performance of the solution.

\begin{figure}[!t]
	\centering
	\includegraphics[width=.95\columnwidth]{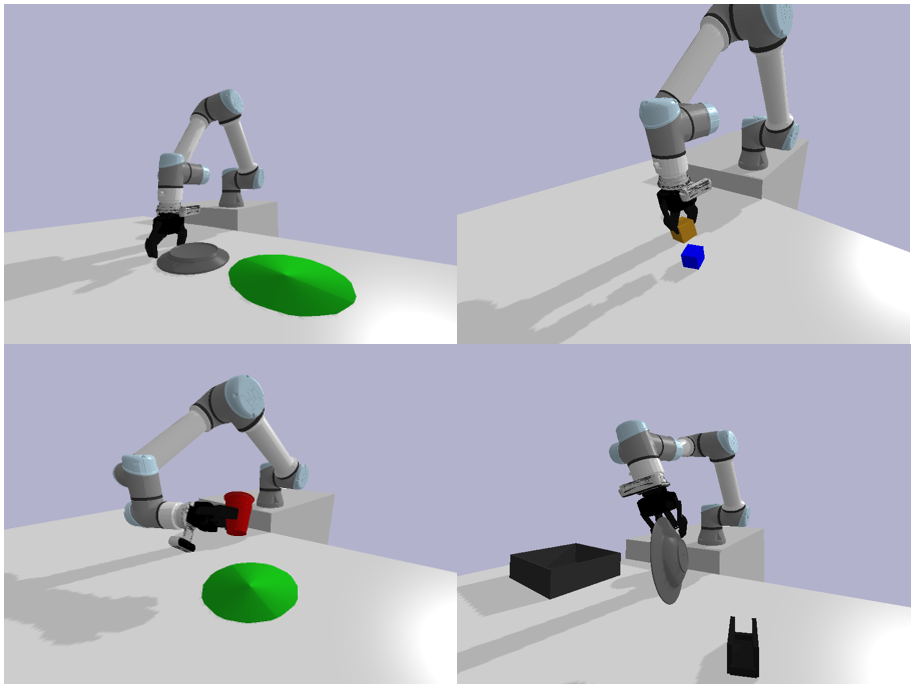}
	\caption{Four object manipulation tasks of pushing plate (top left), stacking box (top right), flipping cup (bottom left), and picking and placing plate (bottom right), where the algorithms are evaluated}\vspace{-3mm}
	\label{fig:envs}
\end{figure}

Therefore, in this paper, to solve the object manipulation with visual input and sparse reward, 
we develop a methods for solving the regularized MDP without the parametric projection, which improves the efficiency and performance compared to the previous approaches.
It is realized by developing a method to directly use optimal non-parametric policy. Through this method, we can unleash the restriction of expressiveness in parametric policy and reduce performance loss arising from the projection step. 
Additionally, we propose a prior distribution that can reflect the prior knowledge of the object manipulation problem. 
Through the capture of the basic structure of the task, this distribution regularizes action space efficiently. Then it  accelerates the training process and improves the final performance.
The proposed framework is evaluated with four object manipulation tasks: pushing plate, stacking box, flipping cup, and picking and placing plate (see Fig. \ref{fig:envs}). The results show that our approach outperforms the state-of-the-art algorithms which have a policy projection step in terms of the sample efficiency and final performance. Furthermore, we present the preliminary sim-to-real result of one task, plate pick-and-place, with a real robot.
Our main contributions are summarized as follows.
\begin{itemize}
    \item Develop RL algorithms to solve the regularized MDP problem with the non-parametric policy
    \item Reveal novel prior distribution which can capture the basic structure of the object manipulation problems
\end{itemize}

The regularized MDP with prior distribution has been widely utilized to encourage exploration \cite{haarnoja2017reinforcement, haarnoja2018soft},
keep trust region \cite{abdolmaleki2018maximum}, and
reduce action space \cite{siegel2020keep, tirumala2020behavior}.
Especially in \cite{siegel2020keep, tirumala2020behavior}, they also utilize the behavior prior for the regularization.
In \cite{siegel2020keep}, 
they train parametric behavior prior from the offline dataset,
and a policy is updated with a MPO-style projection \cite{abdolmaleki2018maximum}.
In \cite{tirumala2020behavior},
they use regularizer to keep basic skills like gait motion,
and a policy is improved with the stochastic value gradients (SVG)\cite{heess2015learning} with the regularization term.
However, both use the parametric policy with a policy iteration scheme, which can restrict expressiveness while making the policy converge to the local minima.

On the other hand, the non-parametric policy is also applied in \cite{kalashnikov2018qt, wang2020critic}. 
In \cite{kalashnikov2018qt}, the cross-entropy method (CEM)\cite{de2005tutorial}
which is a sampling-based optimization method, is used to calculate a deterministic policy from evaluated $Q$ function. However, this approach cannot express the stochastic Boltzmann policy, and it takes a relatively long time to inference because of the iterative CEM process.
In the work of \cite{wang2020critic}, they also use a non-parametric policy with sampling, but it is based on the parametric policy which is updated with parametric projection similar to MPO.

The rest of the paper is organized as follows.
Sec. \ref{sec:prelim} explains preliminary of the problem description and conventional methods to solve it.
In Sec. \ref{sec:Q-learning with Non-Parametric Policy}
and Sec. \ref{sec:Object Manipulation with NPQL},
we propose a novel algorithm which can incorporate the non-parametric policy and a regularizer to improve efficiency and performance of RL with object manipulation.
Experimental details and results for performance comparison and sim-to-real transfer are presents in Sec. \ref{sec:Experiments}.
Finally, Sec. \ref{sec:conclusion} highlights our conclusion and future works.

\section{Preliminary}
\label{sec:prelim}

\subsection{Regularized MDP with Behavior Prior}
\label{sec:regularized_MDP}
Standard infinite horizon MDP is defined in $(\mathcal{S}, \mathcal{A}, r, p_{sa}, \gamma)$ where $\mathcal{S}$ is set of states, $\mathcal{A}$ is set of feasible actions, $r : \mathcal{S} \times \mathcal{A} \rightarrow \left[ r_{MIN}, r_{MAX} \right]$ is reward, $p_{sa} : \mathcal{S} \times \mathcal{S} \times \mathcal{A} \rightarrow \left[ 0 , \infty \right)$ is transition probability conditioned on state $s$ and action $a$, $\gamma \in [0,1)$ is discount factor. The objective of infinite horizon MDP is to maximize cumulative reward $\mathop{\E}_{\tau}{\left[\sum_{t\ge0}\gamma^tr_t\right]}$ where $\tau$ is trajectory. When introducing regularization term $D : \mathcal{S} \rightarrow \mathbb{R}$ and regularizer $b$, the objective is changed to $\mathop{\E}_{\tau}{\left[\sum_{t\ge0}\gamma^t\left(r_t - D\left( \cdot || b \right)\right) \right]}$, and we call MDP with this objective as regularized MDP. This additional regularization can be realized with divergence such as KL divergence, negative entropy, etc.\cite{wu2019behavior}. 

In this works, we want to solve object manipulation by utilizing the prior knowledge as the prior distribution, so we adopt regularized MDP which can incorporate it as regularizer. Specifically, we modify it to parameterize $b$ as behavior prior $b_\theta$, and to have the hard constraint of regularization.
After introducing replay buffer $\mathcal{B}$, the problem becomes,
\begin{equation}
\label{eq:J_pi_theta}
\begin{split}
\max_{\pi\in\Pi_{\theta}, \theta} J(\pi, \theta) = \max_{\pi, \theta}\; & \mathop{\mathbb{E}}_{\substack{s_0\sim\mathcal{B}, \, a_t\sim\pi \\ s_{t+1}\sim p_{s_ta_t}}} \left[ \sum^{\infty}_{t\ge0}{\gamma^t r(s_t,a_t)} \right] \\ 
\textrm{s.t.} \; & D_{KL}\left( \pi_s || b_{\theta,s} \right) \le \epsilon \ \forall s\in\mathcal{B},
\end{split}
\end{equation}
where $\pi_s$ is policy distribution conditioned on $s$, $\Pi_{\theta}=\{\pi_s|D_{KL}\left(\pi_s||b_{\theta,s} \right) \le \epsilon \}$, and $D_{KL}$ is KL divergence. Note that the parameterization of $b$ makes it flexible, allowing it to adapt to progressively evolving $\pi$.

This objective is also recovered from inference formulation as \cite{abdolmaleki2018maximum, levine2018reinforcement, tirumala2020behavior}, and can be solved by joint optimization as expectation-maximization (EM) algorithm. 
\begin{itemize}
    \item \textbf{E-step}: optimize $\pi$ given $\theta_k$ \vspace{-1mm}
    \begin{equation}
    \label{eq:estep_optimization}
        q^k = \argmax_{\pi\in\Pi_{\theta_k}}J\left( \pi, \theta_k \right) \vspace{-1mm}
    \end{equation}
    \item \textbf{M-step}: update $\theta$ given $q^k$ to satisfy \vspace{-1mm}
    \begin{equation}
    \label{eq:mstep_update}
        D_{KL}\left( q^k || b_{\theta_{k+1}} \right) \le D_{KL}\left( q^k || b_{\theta_{k}} \right) \ \text{for} \ \forall s\in \mathcal{B} \vspace{-1mm}
    \end{equation}
\end{itemize}
Here, the update from M-step results in  \vspace{-1mm}
\begin{equation*}
    D_{KL}\left( q^k || b_{\theta_{k+1}} \right) \le \epsilon  \ \text{for} \ \forall s\in \mathcal{B} \vspace{-1mm}
\end{equation*}
with which we can conclude \vspace{-1mm}
\begin{equation*} \vspace{-1mm}
\begin{split}
    J(q^{k+1}, &\theta_{k+1})\\
    &= \max_{\pi\in\Pi_{\theta_{k+1}}}{J(\pi, \theta_{k+1})} \ge J(q^{k}, \theta_{k+1}) = J(q^{k}, \theta_{k})
\end{split}
\end{equation*}
because both $q^k$ and $q^{k+1}$ are in the feasible region of the problem $max_{\pi\in\Pi_{\theta_{k+1}}}{J(\pi, \theta_{k+1})}$.
Therefore the algorithm guarantees monotonic improvement with respect to $J$, and convergence to local optima with bounded reward assumption.
Now let us explain the method to solve each E-step and M-step. The M-step \eqref{eq:mstep_update} is easy to implement because it can be achieved by the supervised learning for a given target distribution $q^k$, with expectation of KL divergence loss over $\mathcal{B}$.
On the other hand, the optimization problem in the E-step \eqref{eq:estep_optimization} can be solved by a method to solve the regularized MDP with fixed prior distribution. Next, prior works to solve E-step is presented.

\begin{table*}[!t]
\renewcommand{\arraystretch}{1.3}
\centering
 \begin{tabular}{|c  c c c c c|} 
 \hline
 \multicolumn{6}{|c|}{\textbf{objective} :
 $ \max_{\pi}{\mathop{\E}_{\pi}\left[ Q\right] - \alpha D_{KL} \left( \pi || b \right) } \quad \text{or} \quad \max_{\pi} \mathop{\E}_{\pi}\left[ Q \right] \ \text{s.t.} \ D_{KL}(\pi||b) < \epsilon $}\\
 \multicolumn{6}{|c|}{\textbf{non-parametric policy} : $q \sim b\exp{Q \over \alpha}$} \\
 \hline\hline
 algorithm & regularizer$(b)$ & policy$(\pi)$ & regularizer update &  policy update & policy evaluation\\
 \hline\hline
 AP-NPQL (ours) & $b_{\theta}$ & q & $\min_{\theta}D_{KL}\left( q || b_{\theta} \right)$ & none & distributional \cite{bellemare2017distributional} \\ 
 \hline
 AP-MPO (sec. \ref{sec:comparison_experiment}) & $b_{\theta}$ & $\pi_{\psi}$ & $\min_{\theta}D_{KL}\left( q || b_{\theta} \right)$ & $\min_{\psi}D_{KL}\left( q || \pi_{\psi} \right)$ & distributional \\ 
 \hline
 AP-SAC (sec. \ref{sec:comparison_experiment}) & $b_{\theta}$ & $\pi_{\psi}$ & $\min_{\theta}D_{KL}\left( \pi_{\psi} || b_{\theta} \right)$ & $\min_{\psi}D_{KL}\left( \pi_{\psi} || q \right)$ & distributional \\ 
 \hline
 ABM-MPO\cite{siegel2020keep} & $b_{\theta}$ & $\pi_{\psi}$ & $\min_{\theta} D_{KL}\left( \pi_{prior} || b_{\theta} \right)$ with ABM & $\min_{\psi}D_{KL}\left( q || \pi_{\psi} \right)$  & TD(0) \\
 \hline
 Alg.1 in \cite{tirumala2020behavior} & $b_{\theta}$ & $\pi_{\psi}$ & $\min_{\theta}D_{KL}\left( \pi_{\psi} || b_{\theta} \right)$ & $\min_{\psi}D_{KL}\left( \pi_{\psi} || q \right) $  & Retrace \cite{munos2016safe} \\
 \hline
 MPO\cite{abdolmaleki2018maximum} & $\pi_{k-1}$ & $\pi_{\psi}$ & $b = \pi_{k-1}$ & $\min_{\psi}D_{KL}\left( q || \pi_{\psi} \right)$ & Retrace  \\
 \hline
 SAC\cite{haarnoja2018soft} & uniform & $\pi_{\psi}$ & none & $\min_{\psi}D_{KL}\left( \pi_{\psi} || q \right) $ & clipped double Q \cite{fujimoto2018addressing}\\
 \hline
\end{tabular}
\caption{Summary of the comparison between algorithms to solve regularized MDP}\vspace{-3mm}
\label{table:comparison_alg}
\end{table*}

\subsection{Parametric Policy Improvement with Regularizer}
\label{sec:parametric_policy}
Now, let us introduce how to solve the regularized MDP or hard constraint version of it \eqref{eq:estep_optimization}.
Most of the algorithms to solve it are based on policy iteration, which approximate Q function and policy as a neural network and progress through the iterative procedure of the policy evaluation and the policy improvement \cite{abdolmaleki2018maximum, wu2019behavior, tirumala2020behavior, siegel2020keep, nair2020accelerating}. The policy improvement methods of the prior works are mainly divided into two categories: MPO\cite{abdolmaleki2018maximum} and SAC\cite{haarnoja2018soft}. 
Both improvement can be viewed as projection of non-parametric policy to parametric policy by minimizing divergence. Given Q function from policy evaluation step, the optimal non-parametric policy can be obtained as $q\propto b \exp{(Q/\alpha)}$, then the objective to project it into parametric policy $\pi_{\psi}$ is $\min_{\psi} D_{KL}\left(q || \pi_{\psi} \right)$, while in SAC, $\min_{\psi}{ D_{KL}\left( \pi_{\psi} || q \right)}$ is used.
This objective is resembled to $\max_{\psi}{ \mathop{\E}_{a\sim\pi_{\psi}}\left[ Q\left(s, a \right) \right] -  \alpha D_{KL}\left( \pi_{\psi} || b\right)}$
where expectation over $\pi_{\psi}$ is calculated by reparameterization.
This is same with SVG\cite{heess2015learning} with additional regularization term, which is adopted in \cite{tirumala2020behavior}. 

The major concern of this paper is that, the parametric projection can degrade performance from non-parametric policy $q$. Thus, we  develop a method to directly use the $q$ for action sampling without any projection, to improve the performance. The summarized differences of ours from others are presented in table \ref{table:comparison_alg}.
Details for how to incorporate the non-parametric policy for the Q-learning framework are explained as follows.

\section{Q-learning with Non-Parametric Policy}
\label{sec:Q-learning with Non-Parametric Policy}

If the policy projection step is removed, the method to solve E-step naturally changes to Q-learning rather than policy iteration. So, in this section, we first introduce Q-learning framework to solve E-step, then show practical way to implement it.

\subsection{Q-learning with Behavior Prior}
Instead of adopting policy iteration by introducing parametric policy as presented in Sec. \ref{sec:parametric_policy},
we solve regularized MDP with Q-learning by introducing regularized Bellman operator.
\begin{theorem}
Define regularized Bellman operator as
\begin{equation}
\begin{split}
\label{eq:operator_T}
    \mathcal{T}Q\left(s,a \right) = r(s,a)+\gamma &\max_{\pi\in\Pi_{\theta_k}}{\mathop{\E}_{s'\sim p_{sa}, a'\sim\pi}\left[Q(s',a')\right]}
\end{split}
\end{equation}
and optimal Q value as 
\begin{equation}
\begin{split}
Q^{\theta_k}(s,a) = &\max_{\pi\in\Pi_{\theta_k}}{\mathop{\E}_{\substack{s_0=s, \ a_0=a\\\tau\sim\pi}}\left[r_t\right]}
\end{split}
\end{equation}
then, $\mathcal{T}^{\infty}Q \rightarrow Q^{\theta_k}$ with converged policy
\begin{equation}
    \label{eq:policy}
    q^k = \frac{b_{\theta_k} \exp{ \left( Q^{\theta_k} / \alpha \right)}} {\mathop{\E}_{a \sim b_{\theta_k}} \left[ \exp{ \left( Q^{\theta_k} / \alpha \right)} \right]}
\end{equation}
where dual variable $\alpha$ of
\begin{equation}
\label{eq:alpha_dual}
 \alpha(s)=\argmin_{\alpha'}\alpha' \epsilon + \alpha'\log{\mathop{\E}_{a\sim b_{\theta_k}}\left[ \exp{\frac{Q^{\theta_k}(s,a)}{\alpha'}}\right] }.
\end{equation}
Furthermore, $q^k$ is optimal solution of $max_{\pi\in\Pi_{\theta_k}}J\left( \pi, \theta_k \right)$.
\end{theorem}

Note that, we treat $\alpha$ as a function of $\mathcal{S} \rightarrow \mathbb{R}$.
With this theorem and stochastic approximation of transition probability, we can solve the optimization problem of \eqref{eq:estep_optimization} by simulated data $(s,a,s')$ as Q-learning \cite{bertsekas1996neuro}. This procedure is differ from methods in \ref{sec:parametric_policy}, in terms of approximated model.
In Q-learning, only the Q function is approximated by a neural network, and the policy is obtained by solving the optimization problem for Q as \eqref{eq:policy}, whereas two neural networks (i.e. policy and Q function) are needed for policy iteration.
Finally, original optimization problem of \eqref{eq:J_pi_theta} can be solved by joint procedure with the E-step of Q-learning and M-step of supervised learning. 

However, to implement this with function approximation such as deep neural network, 
one problem should be resolved, i.e., sampling actions from the non-parametric policy distribution $q^k$.
This policy $q^k$ should be used for the rollout to collect data and the expectation of the Bellman operator \eqref{eq:operator_T}.
As in Sec. \ref{sec:nonparametric_policy}, we resolve this by batch sampling.

\begin{figure*}[!t]
	\centering
    \includegraphics[width=.95\textwidth]{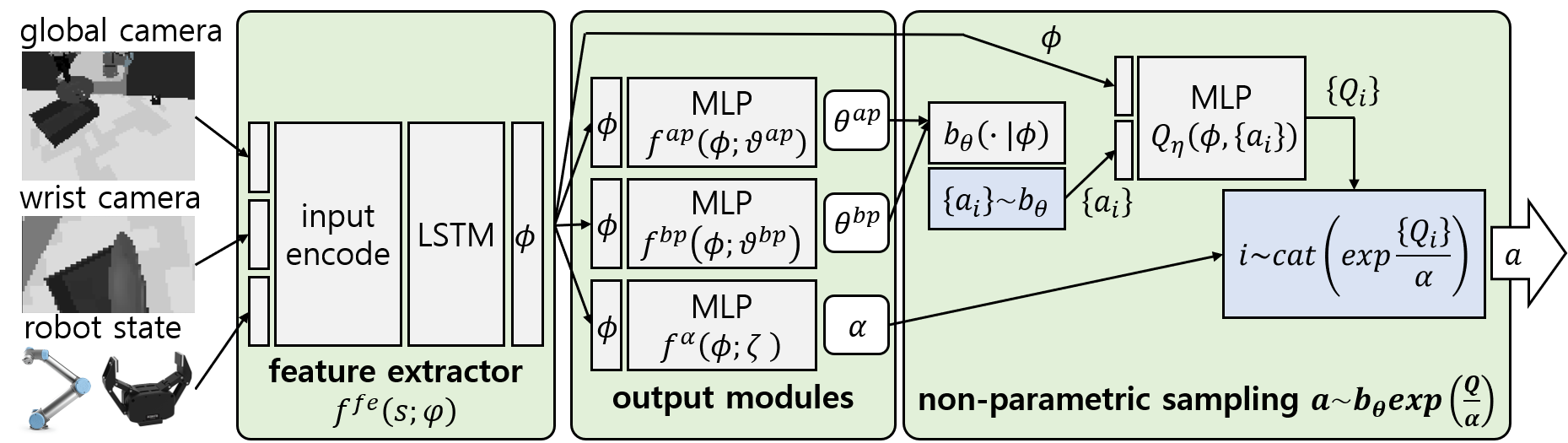}
	\caption{Illustration of network structure used in the experiment and the flow of non-parametric policy. $cat(p_i)$ means categorical distribution with probability proportional $p_i$, and $\{x\}$ means batch data over $x$. The network takes observation of pixel from global and wrist camera, and robot state, while predicting components of non-parametric policy and Q value. The computational efficiency is improve by dividing the network into feature extractor and output modules. The non-parametric policy can express Boltzmann policy over optimal distribution of $a\sim b_{\theta}\exp(Q/\alpha)$ by introducing important sampling from batch data.}\vspace{-2mm}
	\label{fig:network_Structure}
\end{figure*}

\subsection{Non-parametric Policy}
\label{sec:nonparametric_policy}
In continuous setting, sampling from non-parametric distribution have been regarded as intractable\cite{haarnoja2017reinforcement}. Therefore, many algorithms introduced a parametric actor $\pi_{\psi}$ to approximate non-parametric policy with parametric projection as explained in Sec. \ref{sec:parametric_policy}.
Rather than introducing the parametric actor, here we directly use the optimal non-parametric policy $q^k$ of \eqref{eq:policy}, for the collection and the expectation, by utilizing batch sampling.
The approximation of sampling from $q^k$ used in this paper consists of two steps: 1) The batch of action candidates are sampled from parametric behavior prior distribution $b_{\theta_k}$. 2) After evaluating each action candidate with $Q^{\theta_k}$, the action is sampled proportional to $\exp{Q^{\theta_k}\over\alpha}$. This process can be understood as self-normalized important sampling. Through this procedure, Boltzmann policy can be implemented. The illustrative figure is presented in Fig. \ref{fig:network_Structure}.

Expectation over policy $q^k$ is also needed to calculate target value or KL divergence. This can be achieved by using the important sampling:\vspace{-1mm}
\begin{equation*}\vspace{-1mm}
\mathop{\E}_{a\sim q^k}\left[ f(a) \right] 
\approx \sum_{a_i\sim b_{\theta_k}} w_i f(a_i), \ w_i \propto {\exp {Q_i^{\theta_k} \over \alpha}}
\end{equation*}
where $Q_i^{\theta_k} = Q^{\theta_k}\left(s, a_i \right)$ and $w_i$ is calculated with self-normalization within the batch of action samples from $b_{\theta_k}$.
We call Q-learning with this sampling-based non-parametric policy as non-parametric Q-learning (NPQL). 
The brief comparison of this NPQL with other algorithms are presented in table \ref{table:comparison_alg}.
Note that a parametric policy can be introduced by simply inserting additional policy improvement step as described in Sec. \ref{sec:parametric_policy}. 
We implement this approach too, and present comparative experimental results in Sec. \ref{sec:Experiments}.
See Sec. \ref{sec:comparison_experiment} for more details about parametric version.

\subsection{Distributional Q with Batch Action Samples}
\label{sec:Distributional Q with Action Samples}
In the NPQL, we update Q value function by minimizing Bellman residual. Motivated by recent works showing effectiveness of distributional Q to continuous setting \cite{barth2018distributed, vecerik2019practical, ma2020dsac}, we express Q value as categorical distribution with 51 bins of fixed interval \cite{bellemare2017distributional}.
One drawback of the categorical distribution is the truncated Q value, but it can be ignored in our case because the problem we want to solve has a sparse reward with a fixed value, so the range of valid Q value is limited. 

To apply distributional Q to NPQL, the algorithm presented in \cite{bellemare2017distributional} needs to be modified to utilize batch action samples.
Let us start from the minimizing temporal difference (TD) error in terms of KL divergence with incorporation of distribution over both Q value and action,
\begin{equation*}
    \min_{\eta}{\mathop{\E}_{\substack{(s,a,s') \sim \mathcal{B} \\ a'\sim\pi }}\left[D_{KL}\left(\bar{Z}(s,a,s',a') || Z_{\eta}(s,a) \right)\right]}
\end{equation*}
where $Z_\eta$ is distribution over Q value parameterized by $\eta$, $\bar{Z} = r(s,a)+\gamma Z_{\eta'}(s',a')$ is the target value distribution with target parameter $\eta'$.
After introducing categorical representation of $Z$ and batch sampling of action, the objective becomes,
\begin{equation}
    \label{eq:Z_loss}
    \sum_{\mathcal{B}}\sum_{a'_i\sim\pi}{\sum_j{-\bar{z}^i_j \log{z_j}}}
    = \sum_{\mathcal{B}}\sum_j{-\mathop{\E}_{a\sim\pi}{\left[\bar{z}_j\right]} \log{z_j}}
\end{equation}
where $z_j$ is $j$'th bin of $Z_{\eta}(s,a)$ and $\bar{z}^i_j$ is $j$'th bin of $\bar{Z}(s,a,s',a'_i)$.
The expectation over $\pi$ can be calculated with sampling which is explained in Sec. \ref{sec:nonparametric_policy}.
In the \eqref{eq:Z_loss}, calculation results of both sides are the same, but the computation of the right side is more efficient because of the pre-calculation of expectation over action samples.
Practically, we use a large size of action samples, and it substantially improves computational efficiency. 
Note that $Q$ values are recovered by expectation over $Z$, $Q_{\eta}=\mathop{\E}{\left[Z_{\eta}\right]}$. For our experiment, we use N-step backup and this can be simply implemented by changing one-step target $\bar{Z}$ to the N-step target.

\section{Object Manipulation with NPQL}
\label{sec:Object Manipulation with NPQL}
So far, we have explained Q-learning-based algorithm with the non-parametric policy, NPQL. However, to apply NPQL, parametric behavior prior distribution (i.e., $b_\theta$ in \eqref{eq:mstep_update}) is required, from which we can extract action samples. Conditions for proper $b_\theta$ are 1) easy to sample from, 2) supporting the region where Q value is high, 3) not too concentrated for exploration.
Previous works have proposed parametric prior from Variational Auto Encoder (VAE) \cite{fujimoto2019off} or regressing with offline dataset \cite {siegel2020keep}. 
However, these methods do not consider the specific structure of the object manipulation.
In this section, we propose proper prior distribution to reflect this basic structure and complete algorithms by composing it with NPQL.

\subsection{Behavior Prior with Action Primitives}

Our intuition is that, a proper behavior prior can improve the efficiency of RL significantly,
since it can effectively regularize the search space of action.
Therefore, our concern is what we take as the behavior prior for the object manipulation.
Fortunately, we found that the object manipulation task can be executed with a sequence of some basic operations such as moving toward an object and moving toward a goal position while grasping or releasing the gripper.
So, we define the set of those actions as \textit{action primitives} (APs). 
For example, for the plate pick-and-place task in Fig. \ref{fig:AP_figure},
the APs can be: moving toward a plate while releasing the gripper, moving toward a goal while closing the gripper. Specifically, we use joint velocity as input, APs are clipped vector from the current joint state to the target state which can be obtained from inverse kinematics solver.
In the test time, it is hard to access the object pose directly, and we use prediction model of APs $f^{ap}$, which is trained with behavior cloning during training step. We can get reasonably accurate prediction model even with the high dimensional pixel input, as verified in \cite{james2017transferring}.

However, the vector set of APs is not sufficient to construct the distribution of behavior prior ($b_{\theta}$), and we define additional adaptable parameters $\theta^{bp}$, then make Gaussian mixture model (GMM) from them as Fig. \ref{fig:AP_figure}. That is, we set predicted APs ($\theta^{ap}$) as the mean of GMM and make residual GMM parameters ($\theta^{bp}$: mixture weights, covariance matrix) adaptable. 
Then
$\theta^{bp}$ is updated with the M-step \eqref{eq:mstep_update} while $\theta^{ap}$ is frozen.
On the other hand, to improve the exploration while supporting the APs, 
we introduce additional regularization term $\mathcal{H}(b_{\theta})$ to the loss of M-step as
\begin{equation}
\label{eq:behavior_prior_update_objective}
    \mathcal{L}^M = \mathop{\E}_{s\sim\mathcal{B}}
    \left[
    D_{KL}(q^k || b_{\theta})+\nu\mathcal{H}(b_{\theta})
    \right]
\end{equation}
where $\nu$ is a scale factor.
Note we can auto-tune this $\nu$ with the dual formulation of hard constraint, while introducing a hyperparameter $\epsilon_{\nu}$.

\begin{figure}[!t]
	\centering
	\includegraphics[width=.99\columnwidth]{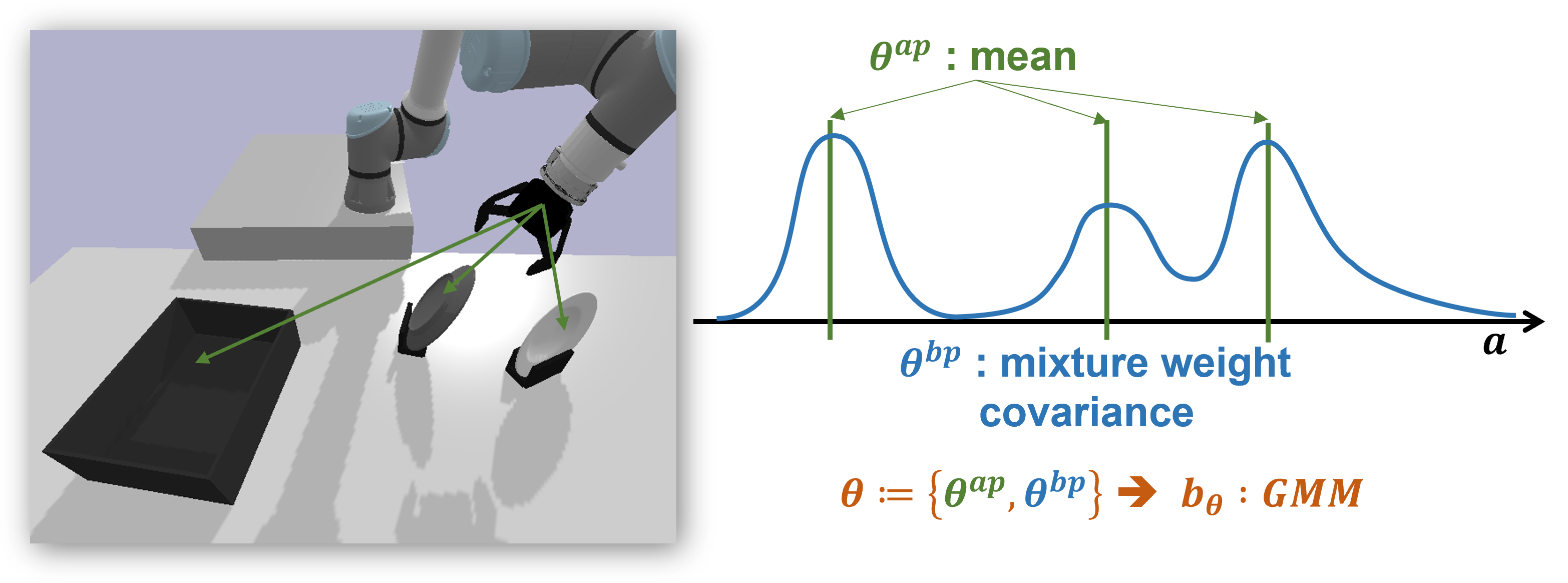}
	\caption{Illustration of APs and distribution of behavior prior. Left snapshot describes one moment between performing picking and placing dishes, and in this case, APs can be defined as vectors toward each dish and a goal. To construct $b_{\theta}$ as GMM, mixture weights and covariance matrix are added, and $\theta^{ap}$ is set as mean of GMM.} \vspace{-2mm}
	\label{fig:AP_figure}
\end{figure}

\subsection{Algorithmic Consideration}
\label{sec:algorithmic_consideration}
The important sampling described in Sec. \ref{sec:nonparametric_policy} can be computationally expensive, since it requires batch calculation.
So, we improve the computational efficiency of inference by carefully designing a network structure. 
The main bottle-neck of NPQL is to evaluate Q values from action samples. 
Naively, we can design network of approximating Q function as $Q(s,a)$, which takes raw observation and action as input. This model is not a problem if the observation is based on low-dimensional state, but if observation includes pixel, the neural network may contains CNN structure and it is not scalable because we should do multiple forward calculation for the number of action samples, for each usage of important sampling (i.e., action selection and Bellman operator). 

So, alternatively, we design the network by dividing it into two parts as Fig. \ref{fig:network_Structure}: feature extractor, output modules. The feature extractor $f^{fe}(s;\varphi)$ is to predict extracted feature $\phi$ from the raw observation $s$. Then, the output modules take $\phi$ as input and predict each desired output. Through this network structure, we can keep extracted feature $\phi$ instead of raw observations, and remove multiple calculation of $f^{fe}$ when evaluating action samples.
Specifically, NPQL has 5 output modules : 
$f^{bp}(\phi;\vartheta^{bp})$ to predict $\theta^{bp}$,
$f^{ap}(\phi;\vartheta^{ap})$ to predict $\theta^{ap}$,
$Z_{\eta}(\phi,a)$ to predict $Z$ as described in Sec. \ref{sec:Distributional Q with Action Samples},
$f^{\alpha}(\phi;\zeta)$ to predict $\alpha$, which is also function of state as \eqref{eq:alpha_dual}. The network structure and batch data flow are described in Fig. \ref{fig:network_Structure}.
Note that, any output module can be added to this structure, e.g. auxililary output to predict pose of the objects.

To combine E-step, M-step and behavior cloning of training APs prediction model, we should consider sequence of each step. Motivated by utilizing auxiliary loss to improve training efficiency \cite{zhu2018reinforcement}, we train behavior cloning and E-step simultaneously, because E-step is in charge of feature extractor training and behavior cloning loss can play a role of auxiliary loss.
Then, total E-step loss becomes,\vspace{-1.1mm}
\begin{equation}\vspace{-1.1mm}
    \label{eq:e_step_loss}
    \begin{split}
    \mathcal{L}^{E}(\varphi, \vartheta^{ap}, \eta) = \lambda^{ap}\mathcal{L}^{ap}(\varphi, \vartheta^{ap}) + \mathcal{L}^{Q}(\varphi, \eta)
    \end{split}
\end{equation}
where $\mathcal{L}^{ap}$ is behavior cloning loss with label of APs, $\mathcal{L}^{Q}$ is calculated with \eqref{eq:Z_loss}, and $\lambda^{ap}$ is scale coefficient for $\mathcal{L}^{ap}$.
In the M-step, we freeze $\varphi$ and train only $\vartheta^{bp}$ with the loss \eqref{eq:behavior_prior_update_objective}. Total algorithm is presented in algorithm \ref{algorithm:NPQL}.

\begin{algorithm}[!t]
\SetAlgoLined
 initialization network parameters and replay buffer $\mathcal{B}$\\
 \For{each iteration}{
 \For{each environment step}{
    $\mathcal{B} \leftarrow \mathcal{B} \cup BatchRollOut()$
 }
 \For{each gradient step}{
     $\phi = f^{fe}(s;\varphi)$ \\ 
     $\theta^{bp}=f^{bp}(\phi;\vartheta^{bp}), \ \theta^{ap}=f^{ap}(\phi;\vartheta^{ap})$ \\ 
     $\{a, \log{b_{\theta}(a)}\} \sim b_{\theta}$\\ 
     \textbf{E-step} with \eqref{eq:e_step_loss}: $\nabla\mathcal{L}^{E}\left(\varphi, \vartheta^{ap}, \eta \right)$ \\
     $\alpha$ update with \eqref{eq:alpha_dual} \\
     \textbf{M-step} with \eqref{eq:behavior_prior_update_objective}: $\nabla\mathcal{L}^{M}\left( \vartheta^{bp} \right)$ \\
     $\nu$ update\\
     target update
 }
 }
 \caption{AP-NPQL}
 \label{algorithm:NPQL}
\end{algorithm}

Because NPQL is the off-policy algorithm, transition data from any behavior policy can be used. Additionally, we can design planning of performing tasks utilizing accurate state data of objects, because we apply sim-to-real approach in this paper and training procedure is performed in the simulation.
So, we make an expert dataset from scripted policy, then keep it as a split replay buffer. That is, we maintain two replay buffer and sample from both buffer with fixed ratio motivated by \cite{paine2019making}.

\begin{figure*}[t]
\centering
\subfloat[pushing plate]{
\includegraphics[width=.25\textwidth]{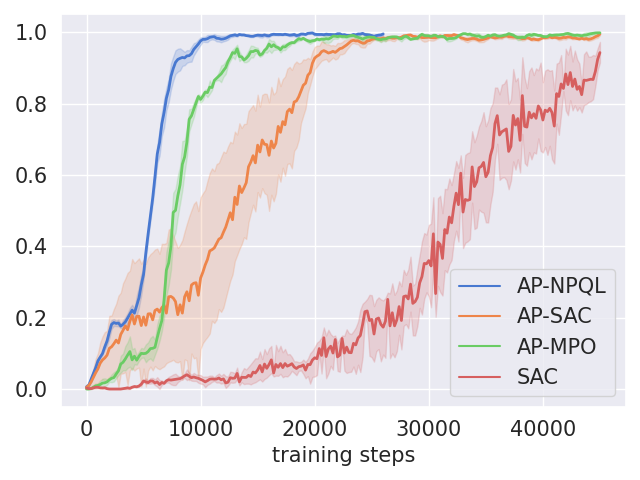}
}
\hspace{-5mm}
\subfloat[stacking box]{
\includegraphics[width=.25\textwidth]{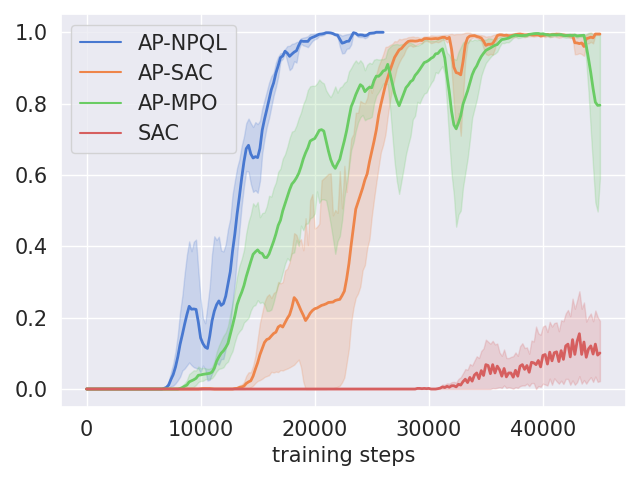}
}
\hspace{-5mm}
\subfloat[flipping cup]{
\includegraphics[width=.25\textwidth]{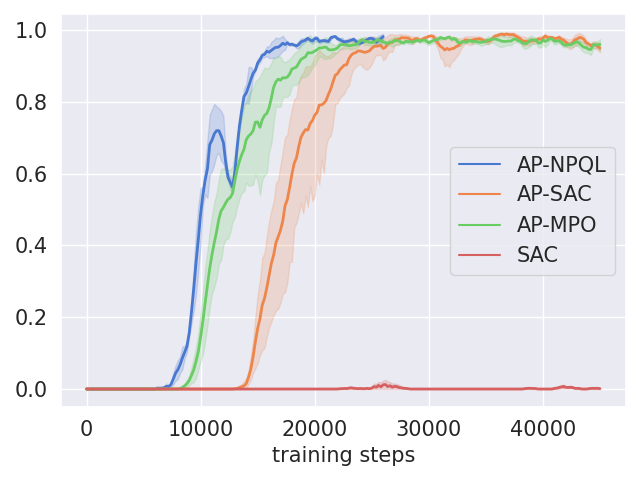}
}
\hspace{-5mm}
\subfloat[picking and placing plate]{
\includegraphics[width=.25\textwidth]{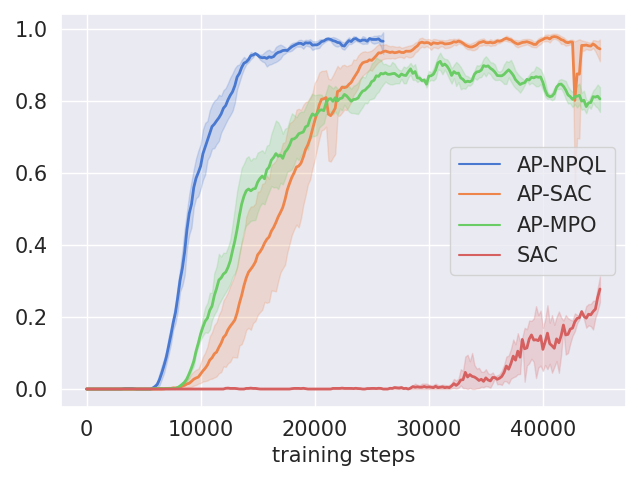}
}
\caption{The plots to show training progress. The horizontal axis is number of training steps, and the vertical axis is success rate.
All algorithms are trained with replay buffer from scripted policy, and the algorithms with prefix 'AP-' are trained with behavior prior based on APs. The plots are drawn from three different training of each algorithms, with each different random seed.
It clearly shows that the training efficiency is improved by introducing APs (AP-NPQL, AP-SAC, AP-MPO vs. SAC),
and the proposed AP-NPQL shows better final performance in picking and placing plate, due to the elimination of the parametric projection.
Note that the baseline SAC even fails for the flipping cup task.
}\vspace{-2mm}
\label{fig:comparative_results}
\end{figure*}

\section{Experiments}
\label{sec:Experiments}
AP-NPQL is specialized with object manipulation task, and the performance of it is evaluated with four object manipulation tasks. Through the comparison experiments with other algorithms, we show the effects of introducing non-parametric policy and APs. Details about the experiment setting and results are explained in this section.

\subsection{Environments}
\label{sec:environment}
Across all environments including real-world setup, we utilize UR5e 6DOF manipulator, equipped with a Robotis 2 finger parallel gripper (RH-P12-RN), Intel RealSense Camera D435 for wrist camera, and D455 for the global camera. The global camera views the whole scene and the wrist camera is attached to the last joint to see local details which can be occluded in the global camera.
Pybullet \cite{coumans2019} is used to construct simulation environments. 

For the validation and comparative experiment, four object manipulation tasks as in Fig. \ref{fig:envs}, are employed in this paper:
\begin{itemize}
    \item \textbf{Pushing plate}: The robot should push a plate on the table to the goal and get a reward of +1 if the plate is on the goal. 
    \item \textbf{Stacking box}: The robot should stack one box on another. The colors of the boxes are fixed with yellow and blue. The reward is granted only when boxes are stacked. 
    \item \textbf{Flipping cup}: The robot should flip a cup and move it to the goal. The shape of the cup is randomly selected within three shapes. The robot gets a reward of +1 if the cup is on the goal and flipped.
    \item \textbf{Picking and placing plate}: A plate is standing on a holder and the robot should pick the plate and move it to the goal box. The reward +1 is granted only when the plate is into the goal box. 
\end{itemize}

The action is desired velocity for each six joint and gripper movement of $\{open, close\}$ with the period of 0.15 seconds. The velocity of manipulator is limited to $[-1,1] m/s$, and the acceleration is limited to $[-1.2, 1.2] m/s^2$, and the gripper moves with a fixed velocity given discrete action. The observations include pixel streaming from two cameras, global of size 72$\times$96 and wrist of size 36$\times$64 as figure \ref{fig:network_Structure}.
For the robot states, we use the 6DOF arm joint angles and the 1DOF gripper angle. In addition, we stack the last three results for these robot states.
The trajectory is limited to 14 seconds. 
For all the above four object manipulation tasks, we define three APs, in joint configuration space $\in \mathbb{R}^6$, which correspond to a target object, the desired goal, and a fixed neutral configuration, respectively.
We can easily obtain these APs by calculating the difference between the desired configuration and the current configuration of the manipulator.
The object and the goal are spawned uniformly with a fixed range. 

\subsection{Comparison Experiments}
\label{sec:comparison_experiment}
We execute a comparison experiment to verify the benefits of introducing non-parametric policy and APs.
To verify them, we implement ourselves parametric version algorithms with behavior prior based on APs.
If we introduce parametric policy, Q-learning setting in Sec. \ref{sec:Q-learning with Non-Parametric Policy} is changed to policy iteration. For that, Bellman operator should be changed to policy evaluation as\vspace{-1mm}
\begin{equation*}\vspace{-1mm}
    \mathcal{T}^{\pi_{\psi_k}} Q\left( s,a \right) = r\left(s,a\right)+\gamma \mathop{\E}_{s'\sim p_{sa}}\left[\mathop{\E}_{a'\sim\pi_{\psi_k}}{\left[Q(s',a')\right]} \right],
\end{equation*}
and policy improvement step should be added to update the parameters of the policy with methods of MPO and SAC as explained in Sec. \ref{sec:parametric_policy}.
We call the algorithm implemented by each projection method AP-MPO and AP-SAC, respectively.

Let us explain detail set-up used in experiments. Regarding network design, for AP-NPQL and the critic network of AP-MPO and AP-SAC, we stack two layers with 256 width for $f^{bp}$, $f^{ap}$, $Z_{\eta}$, and $f^{\alpha}$, and use simple two layers convolution network with relu activation for encoding pixel observation from global and wrist camera. The policy network of AP-MPO, AP-SAC, and SAC have same structure except that they only predict actions.
Discount factor $\gamma$ is 0.99, E-step learning rate is 3e-4, $\alpha$ learning rate is 1e-4, M-step learning rate is 3e-4, $\nu$ learning rate is 1e-4, delay rate for target update $\lambda$ is 0.05, sampling size to calculate target Q-value is 120, sampling size to Boltzmann policy is 100, KL divergence limit $\epsilon$ is tuned by grid search from 0.5 to 4.0 with 0.5 interval, and $\epsilon_\nu$ is also tuned by grid search.
We utilize 10 parallel simulation with GPU of NVIDIA Tesla V100.

Total 4 algorithms (AP-NPQL, AP-MPO, AP-SAC, SAC) are applied to four object manipulation environments in simulation as described in \ref{sec:environment} with additional replay buffer of expert data as in \ref{sec:algorithmic_consideration}, and the results are shown in Fig. \ref{fig:comparative_results}. First, the effect of introducing APs is shown with a comparison between SAC and others. As seen in the plot of all environments, the SAC is an order magnitude slower than other methods and has lower performance. This shows that introducing APs can help to find a solution efficiently.

The effects of parametric projection are revealed in the comparison plot between AP-NPQL, AP-MPO, and AP-SAC. Three methods use the same prior distribution, but the final performance and the convergence speed are different. 
Across all four environments, the result of AP-MPO converges rapidly at the beginning, but the final performance is relatively poor, while the result of AP-SAC progress slowly, but the final performance is better than that of AP-MPO.
The reason for this is related to the properties of KL divergence.
The projection of MPO is to use forward, moment-matching KL divergence, then it has the strength to quickly adopt to non-parametric distribution, but in the final stage of the training, it is disturbed to match the local peek.
On the other hand, the projection in SAC is to use reverse, mode-seeking KL divergence, then it can delicately find the local peek distribution, but in the first stage, it is hard to find an effective region. So empirically, AP-SAC is more sensitive to hyper-parameters than AP-MPO.
Finally, AP-NPQL has no projection step, therefore it can not only quickly find but also delicately describe peek regions.

\subsection{Sim-to-real Transfer to Real Robot}
\label{sec:exp_sim_to_real}
As shown in Fig. \ref{fig:sim-to-real}, we transfer the result and measure the performance of the plate pick-and-place task to the real hardware environment, which is the most difficult task among above four tasks due to the various plate pose.
For this, we use simple domain randomization technique, for both visual and physical domains\cite{alghonaim2020benchmarking, tobin2018domain}.


We randomize visual streaming, camera pose, table height, object size, object shape, table orientation, and gripper speed within fixed range. For randomization of visual streaming, we change the color and texture of the robot and objects as in Fig. \ref{fig:sim-to-real}. The pixel size of the sim-to-real experiment is increased to 120$\times$160.
In both simulation and real world, we use the acceleration control mode, where the desired acceleration is calculated from the difference between desired and current velocity, and the direct torque inputs to each joint are calculated from inverse dynamics. We use this control mode to enhance interaction between the environment, including table and objects, and the manipulator.
In the experiment of real hardware, we use admittance control in which virtual simulation used in training is utilized, and it can reduce the physical sim-to-real gap.

\begin{figure}[t]
\centering
\subfloat{
\includegraphics[width=.96\columnwidth]{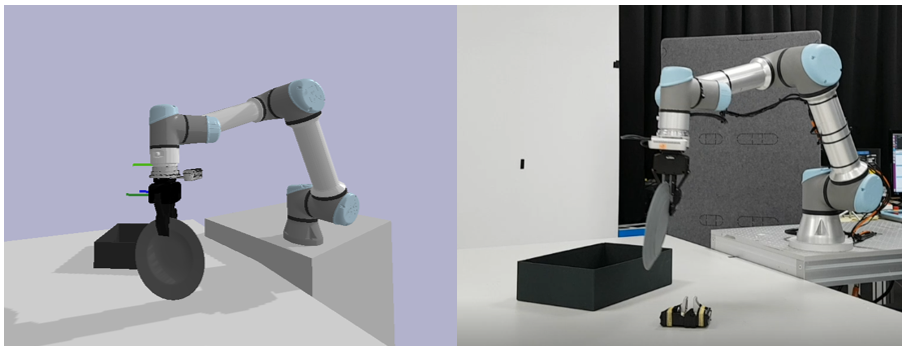}
}
\vspace{-0.2pt}
\subfloat{
\includegraphics[width=.953\columnwidth]{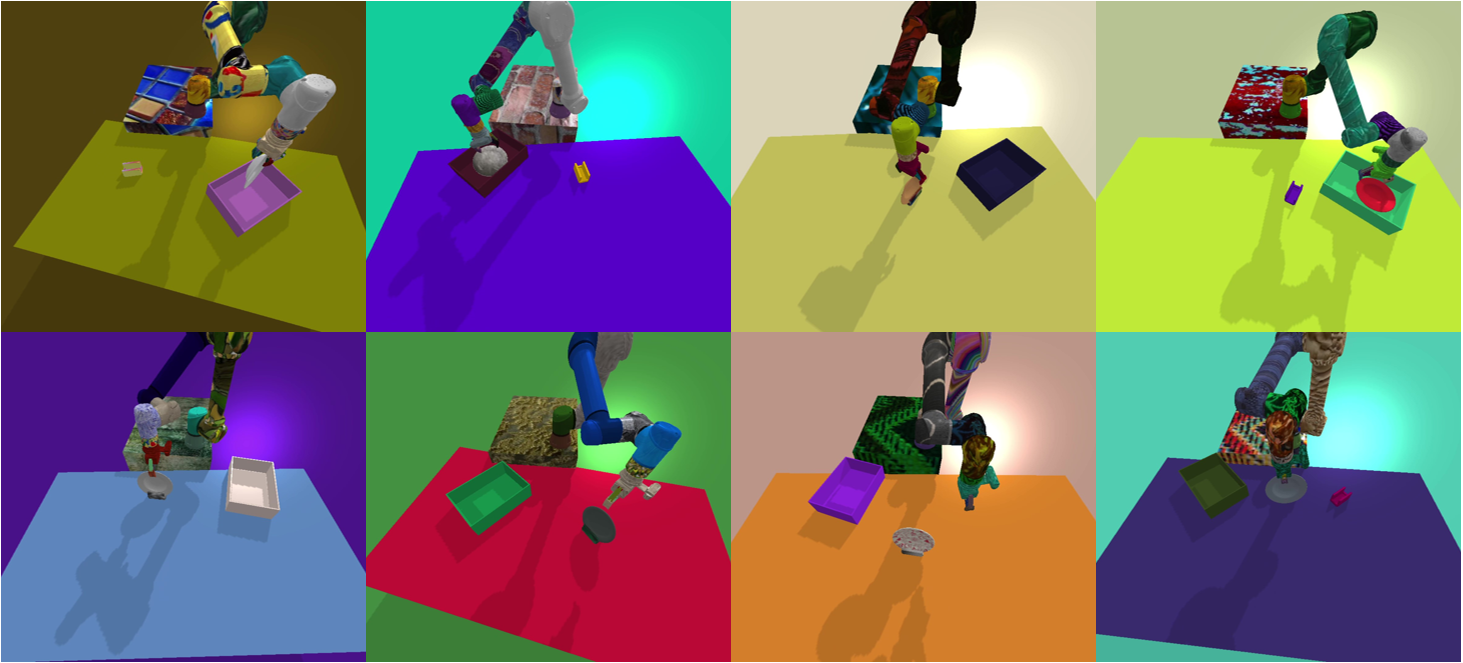}
}
\caption{Snapshots of performing plate pick-and-place task in the simulation (top left) and real world (top right), and variants of randomization (bottom). The policy can decide actions from raw sensory inputs
as closed-loop.
}\vspace{-3mm}
\label{fig:sim-to-real}
\end{figure}
 
We evaluate transfer performance by measuring success rate in the real setup as Fig. \ref{fig:sim-to-real}. The real environment is replicated with the simulation setup, including camera position, initial configuration, initial object and goal pose, the color of the table, and gripper speed. The transferred policy can succeed in tasks with a rate of 70\% over 20 attempts. The performance drops through transfer, and it results from various reasons. We found that the policy is sensitive to distractors such as line and background. We expect that the performance can be improved with adversarial approach \cite{rao2020rl}, utilizing depth camera \cite{xiang2020learning}, but we remain it as future works.
Qualitatively, the trained policy can chase plate even with human interaction, and perform the task with various colors of the plate, from raw sensory inputs as closed-loop.

\section{conclusion}
\label{sec:conclusion}
In this paper, we present an efficient RL formulation to solve the object manipulation problem with visual input and sparse reward. We first formulate the object manipulation as an inference problem with a prior and show that it can be solved by joint EM-style optimization. Then we propose the NPQL, which is based on regularized Q-learning algorithm, while dominant prior works are based on actor-critic, which cause a performance loss by additional parametric projection step. We can gain additional performance and efficiency benefits by introducing action primitives as a basis of the behavior prior. The combined framework AP-NPQL is evaluated in four object manipulation tasks, and one of them, plate pick-and-place, is transferred to the real robot in a sim-to-real manner.

An interesting future direction is to combine NPQL with various behavior prior models. As in \cite{fujimoto2019off, siegel2020keep}, we can train parametric behavior prior model from batched data $\mathcal{B}$ and it naturally extends to offline reinforcement learning. 
Another direction is to extend to multi-modal tasks with structured behavior prior such as \cite{zhuang2020lyrn}.

\bibliographystyle{IEEEtran}
\bibliography{IEEEabrv, ref}

\appendix
\subsection{Contraction of $\mathop{T}$}
\label{adx:contraction_T}
\begin{proof}
    Define $\lVert Q_1-Q_2 \rVert=\max_{a,s}\lvert Q_1(s,a)-Q_2(s,a) \rvert$, then
    for any $Q_1, Q_2$,
    \begin{align*}
        \lvert \mathop{\E}_{\pi_1^*}Q_1 - \mathop{\E}_{\pi_2^*}Q_2 \rvert &\le \max_{\pi \in (\pi_1^*, \pi_2^*)}\lvert \mathop{\E}_{\pi}Q_1 - \mathop{\E}_{\pi} Q_2 \rvert \\
        &\le \lVert Q_1 - Q_2 \rVert
    \end{align*}
    where $\pi_n^*=\argmax_{\pi}\mathop{\E}_{\pi}Q_n \ \text{s.t.} \ D_{KL}(\pi || b) \le \epsilon$. Then,
    \begin{align*}
        \lVert \mathcal{T}Q_1 -  \mathcal{T}Q_2 \rVert &= \gamma \lVert \mathop{\E}_{s'\sim p_{sa}}{\left[ \mathop{\E}_{\pi_1^*}Q_1 -  \mathop{\E}_{\pi_2^*}Q_2 \right]} \rVert \\
        &\le \gamma \lVert \mathop{\E}_{s'\sim p_{sa}}{\lVert Q_1 - Q_2 \rVert} \rVert \\
        &=\gamma \lVert Q_1 - Q_2 \rVert
    \end{align*}
\end{proof}

\end{document}